\documentclass[10pt,letterpaper]{article}

\usepackage{cogsci}
\cogscifinalcopy

\usepackage[style=apa,natbib=true,annotation=false]{biblatex}
\usepackage{float}

\usepackage{hyperref}
\usepackage{cleveref}
\crefname{thm}{theorem}{theorems}
\Crefname{thm}{Theorem}{Theorems}
\crefname{lem}{lemma}{lemmas}
\Crefname{lem}{Lemma}{Lemmas}

\usepackage{amsmath}
\usepackage{dsfont}

\usepackage{amsthm}
\usepackage{thm-restate}
\usepackage{mathtools}
\usepackage{xparse}
\usepackage{xfrac}

\DeclareMathOperator*{\E}{\mathchoice
  {\scalebox{1.5}{\ensuremath{\mathbb{E}}}}
  {\mathbb{E}}
  {\mathbb{E}}
  {\mathbb{E}}
}
\DeclareMathOperator*{\Var}{\mathchoice
  {\scalebox{1.2}{\ensuremath{\mathrm{Var}}}}
  {\ensuremath{\mathrm{Var}}}
  {\ensuremath{\mathrm{Var}}}
  {\ensuremath{\mathrm{Var}}}
}

\NewDocumentCommand{\p}{ o o }{%
  \ensuremath{\pi^{(\IfNoValueTF{#2}{N}{#2})}_{s_{\IfNoValueTF{#1}{i}{#1}}}}%
}

\NewDocumentCommand{\pj}{ o o }{%
  \ensuremath{\pi^{(\IfNoValueTF{#2}{N}{#2})}_{\IfNoValueTF{#1}{j}{#1}}}%
}

\NewDocumentCommand{\m}{ o o }{%
  \ensuremath{M^{(\IfNoValueTF{#2}{N}{#2})}_{s_{\IfNoValueTF{#1}{i}{#1}}}}%
}

\theoremstyle{plain}

\newtheorem{lem}{Lemma}

\theoremstyle{definition}

\allowdisplaybreaks

\title{Algorithmic Consequences of Particle Filters for Sentence Processing:\\ Amplified Garden-Paths and Digging-In Effects}

\author{\mbox{Amani Maina-Kilaas (amanirmk@mit.edu)}}
\author{\mbox{Roger Levy}}
\affil{Department of Brain and Cognitive Sciences, MIT}

\begin{document}

\maketitle

\begin{abstract}
Under surprisal theory, linguistic representations affect processing difficulty only through the bottleneck of surprisal. Our best estimates of surprisal come from large language models, which have no explicit representation of structural ambiguity. While LLM surprisal robustly predicts reading times across languages, it systematically underpredicts difficulty when structural expectations are violated---suggesting that representations of ambiguity are causally implicated in sentence processing. Particle filter models offer an alternative where structural hypotheses are explicitly represented as a finite set of particles. We prove several algorithmic consequences of particle filter models, including the amplification of garden-path effects. Most critically, we demonstrate that resampling, a common practice with these models, inherently produces real-time digging-in effects---where disambiguation difficulty increases with ambiguous region length. Digging-in magnitude scales inversely with particle count: fully parallel models predict no such effect.

\end{abstract}

\section{Introduction}

How does the mind maintain expectations about syntactic structure as sentences unfold in real time? This question lies at the heart of psycholinguistics, with deep implications for theories of language processing. One dominant framework---surprisal theory---proposes that processing difficulty reflects how much a word forces us to update our expectations. Yet this theory assumes that comprehenders maintain all possible structural interpretations simultaneously---an assumption questioned by evidence that fully parallel models underpredict the difficulty of structural violations. Particle filter models offer a potential solution, providing a ``rational approximation'' that represents structural uncertainty through a finite set of hypotheses. Here, we prove that limited parallelism amplifies garden-path effects, and that resampling---a common practice with these models---causes words that disambiguate toward dispreferred structures to become harder to process over time, without any new information. Understanding that particle filters predict such digging-in effects---typically associated with dynamical systems models---is crucial for evaluating them as cognitive architectures for human sentence processing.

\subsection{Surprisal Theory}

Surprisal theory is an expectation-based, resource-allocation theory of sentence processing difficulty \citep{hale-2001-probabilistic, Levy2008ExpectationbasedSC}. Imagine taking all possible structural analyses consistent with a partial sentence and ranking them preferentially according to a probability distribution. With each new word, invalidated analyses are discarded and their probability mass reallocated. Cognitive effort corresponds to this change in the distribution, as measured by relative entropy. Equivalently, effort is proportional to the word's information content, measured by \textbf{surprisal}---the negative log-probability of the word in context. Large language models (LLMs) operationalize this theory, implicitly marginalizing over all possible structures to predict the next word. While LLM surprisal robustly predicts reading times in diverse languages \citep{wilcox-etal-2023-testing, cory2024large}, it systematically underpredicts difficulty when structural expectations are violated \citep{Huang2024LargescaleBY,wilcox-etal-2021-targeted,vanSchijndel2020SingleStagePM}. The underprediction suggests that representations of ambiguous syntactic structure are causally implicated during language processing in a way that LLMs do not capture. Although the standard surprisal theory invokes total parallelism, it is acknowledged that such parallelism is computationally expensive and may not be psychologically plausible \citep{Jurafsky1996APM,Levy2008ExpectationbasedSC}, suggesting the possibility of a limitedly parallel parser where multiple but not all analyses are maintained. Here, particle filter models offer a compelling alternative to LLMs: they approximate the full distribution using a finite set of analyses, providing a natural account of how resource-limited comprehenders might achieve near-rational performance.

\subsection{Particle Filter Models}

A particle filter represents uncertainty about the current state of the world---in our case, the syntactic structure of a sentence---using a collection of $N$ \textbf{particles}, each corresponding to a hypothesis about that state \citep{Gordon1993NovelAT}. The key idea is to recursively update a set of samples from a distribution as more data are obtained. Particle filters belong to the family of Monte Carlo methods, which have made it possible to solve problems that were previously intractable \citep{Doucet2001SequentialMC}, and can be considered ``rational approximations to rational models'' in cognitive science \citep{Sanborn2010RationalAT}.

Applied to sentence processing, each particle represents a structural analysis of the sentence so far and carries a weight reflecting its current plausibility. As new words arrive, the model updates its beliefs through a simple cycle: (1) weights are updated based on how well each particle's structure predicted the observed word, (2) particles are resampled with replacement according to their weights, and (3) weights are reset to uniform. The resampling step converts a weighted set of particles into a fresh unweighted sample from the same distribution. With finite particles, sampling variance can shift the distribution and eliminate structures entirely.

This framework has several properties that make it attractive as a limitedly parallel implementation of surprisal theory: the algorithm is inherently incremental and can yield word-by-word surprisal estimates, ambiguous syntactic structure is explicitly represented in particles, and the number of particles $N$ provides a natural parameter for memory limitations.

\subsection{Algorithmic Consequences}

Understanding what candidate algorithms predict is crucial for evaluating cognitive theories and constraining models of human sentence processing. ``Rational'' models---such as surprisal theory---are distinguished from ``dynamical systems'' models \citep{Tabor2004EvidenceFS,TABOR1999491,TABOR2004355} in at least one important regard: the latter have a feedback mechanism where preferred analyses continue to grow in strength over time. This predicts \textbf{digging-in effects} \citep{Tabor2004EvidenceFS}, where disambiguation difficulty increases with the length of the ambiguous region. This effect is not naturally predicted by surprisal theory; in fact, \citet{futrell-etal-2019-neural} found that recurrent neural networks predicted greater difficulty with shorter ambiguous regions. However, \citet{Levy2008ModelingTE} demonstrated that a particle filter can provide a rational account of digging-in: dispreferred analyses may be deleted at each resampling step and, if the correct analysis is no longer present at disambiguation, the parser would fail. Their model successfully predicted retrospective difficulty ratings by linking them to parse failure probability.

Yet for surprisal theory---a framework fundamentally concerned with incremental difficulty---\textit{real-time} digging-in effects are the critical test case. \citet{Levy2008ModelingTE} leaves open a crucial question: if comprehenders successfully parse the sentence without failure, should longer ambiguous regions still produce greater processing difficulty at disambiguation? Here we mathematically demonstrate several algorithmic consequences of particle filter models---most critically, that resampling produces real-time digging-in effects.\\\par

We begin by introducing the formal setup for our analysis, then present the main conceptual results. We subsequently provide the theoretical analysis for interested readers and conclude by discussing implications for models of sentence processing. Proofs are available in the appendix of the paper.

\section{Formal Setup}

Consider a sentence context $C$: under any probabilistic model, the distribution over possible next words $w$ is \mbox{$P(w\mid C)$}. If possible grammatical structures \mbox{$T \in \mathcal{T}$} mediate the relationship between this context and the next words, then we can make this explicit via marginalization:
\begin{equation}
    P(w\mid C) = \sum_{T \in \mathcal{T}}P(w\mid T,C)P(T\mid C).
\end{equation}

Now assume we model this using a particle filter, which estimates the probability of a word \mbox{$M(w\mid C)$} using learned distributions \mbox{$\pi(T\mid C)$} over structures and \mbox{$Q(w\mid T,C)$} over words given structures:
\begin{equation}
    M(w\mid C) \coloneqq \sum_{T\in \mathcal{T}} Q(w\mid T,C)\pi(T\mid C).
\end{equation}
In practice, the model samples a set of $N$ particles---representing structures---from \mbox{$\pi(T\mid C)$}, producing an approximation \mbox{$\pi^{(N)}(T\mid C)$} and corresponding \mbox{$M^{(N)}(w\mid C)$}. After observing a new word and updating particle weights, particles are then resampled with replacement. Here, we make a strong simplifying assumption: each word in the lengthened context provides \textit{no additional information}---that is, $Q$ and $\pi$ remain fixed. This assumption isolates the resampling mechanism from language statistics, allowing us to show that digging-in effects are inherent to the process itself. Under this assumption, the only effect of lengthening the context is invoking resampling steps. Borrowing the notational convention of \citet{Fearnhead2010RandomweightPF}, we denote the model after $i$ steps of resampling with $N$ particles as:
\begin{equation}
 \m(w\mid C) \coloneqq \sum_{T\in \mathcal{T}} Q(w\mid T,C)\p(T\mid C).
\end{equation}
Throughout this work, our quantity of interest is the expected\footnote{In this work, the expectation subscript denotes the variable being averaged: \mbox{$\E_{\p}\left[S\left(\p\right)\right] = \sum_{\p} S\left(\p\right) P\left(\p\right)$}.} surprisal over possible particle sets \mbox{$\E_{\p}\left[S\left(\p\right)\right]$}, where \mbox{$S\left(\p\right) \coloneqq -\log\m\left(w\mid C\right)$} is the surprisal for a given word in context estimated using the current set of particles $\p$. In a behavioral experiment, the expected surprisal is the quantity that should correspond to the average response time from a large sample of participants, under surprisal theory.

\section{Main Results}

Here, we present our central findings in accessible terms. The formal mathematical statements underlying these interpretations are provided in the following section.

\subsection{Limited Parallelism Increases Expected Surprisal}

Our first result establishes that limited parallelism, in and of itself, increases expected surprisal relative to the true distribution. The only exceptions occur when the word carries no structural information or only a single structure is maintained.

Consequently, particle filter models should show magnified surprisal effects compared to fully parallel models. The top row of \Cref{fig:dig} illustrates this effect: expected surprisal grows as particles $N$ decreases. This result derives from \Cref{thm:surprisal_under_resampling}.

\subsection{Limited Parallelism Amplifies Garden-Path Effects}

Building on the previous result, we find that expected surprisal increases more for words disambiguating toward dispreferred structures. The increase is larger when disambiguation is strong (structures assign very different probabilities to the word) and the word's overall probability is low (the most likely structures assign very low probability).

For garden-path sentences, this means that expected surprisal increases more in ambiguous contexts, amplifying the difference between conditions. The bottom row of \Cref{fig:dig} shows how garden-path effects grow as the particle count $N$ decreases. This result derives from \Cref{thm:second_order_delta,thm:linear_diffusion_delta}.

\subsection{Resampling Produces Real-Time Digging-In Effects}

The prior results describe what happens when a new distribution is formed by sampling finite particles from a previous one. Resampling repeats this process, with each step monotonically increasing expected surprisal and magnifying garden-path effects. \Cref{fig:resampling} illustrates how the expected surprisal evolves over successive resampling steps when disambiguating to preferred or dispreferred structures.

This may seem counterintuitive: resampling is a neutral drift process, so why should next-word predictions change? In fact, they do not---the \textit{probability} of each word remains constant in expectation. However, variance in its estimate combines with the nonlinear linking function of surprisal to produce systematic increases.

The consequence is digging-in effects. Under our assumption of uninformative ambiguous regions, longer regions induce more resampling steps and thus larger garden-path effects. The bottom row of \Cref{fig:dig} shows that digging-in is predicted whenever multiple but finite particles are used. This result derives from \Cref{thm:second_order_delta,thm:linear_diffusion_delta}.

\begin{figure}[htbp]
    \centering
    \includegraphics[width=0.49\linewidth]{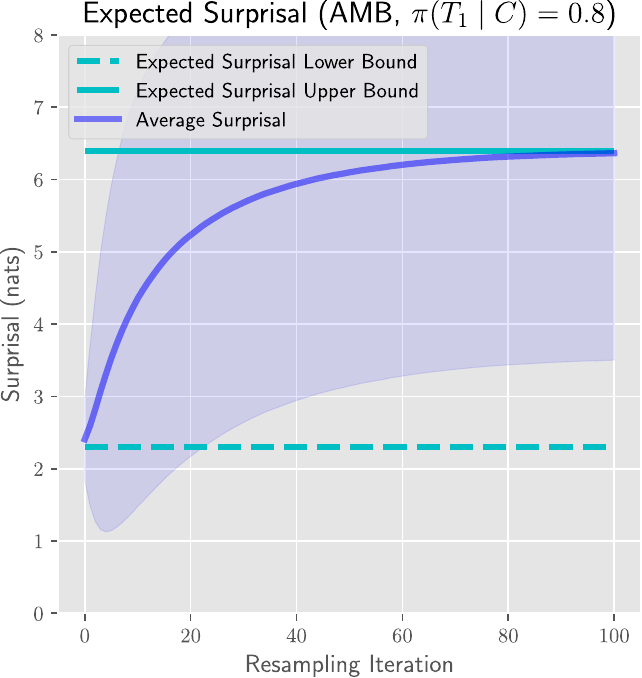} \hfill
    \includegraphics[width=0.49\linewidth]{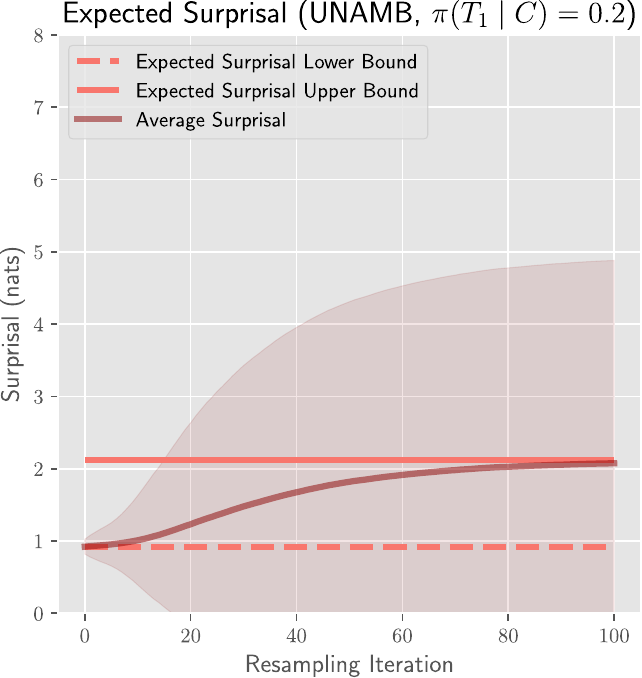}
    \caption{Expected surprisal during resampling. Assume only two structures, $T_1$ and $T_2$, which give word $w$ the in-context probability \mbox{$Q(w\mid T_1,C) = 0.004$} and \mbox{$Q(w\mid T_2,C) = 0.5$}, such that $w$ strongly disambiguates to $T_2$. Using 25 particles, we simulate the change in surprisal with a context that prefers $T_1$, \mbox{$\pi_{\text{AMB}}(T_1\mid C) = 0.8$}, and a context that prefers $T_2$, \mbox{$\pi_{\text{UNAMB}}(T_1\mid C) = 0.2$} (\textit{not unambiguous in typical sense}). Expected surprisal is lower-bounded by $S\left(\pi\right)$, reflecting full parallelism, and upper-bounded by $\E_{\p[\infty]}\left[S\left(\p[\infty]\right)\right]$, at which point only a single structure is entertained by the parser. Shaded regions indicate \textpm1 stdev of the sample (50 thousand trials); 95\% error bars would not exceed the line width.}
    \label{fig:resampling}
\end{figure}

\begin{figure}[htbp]
    \centering
    \vspace{1.1mm}
    \includegraphics[width=0.49\linewidth]{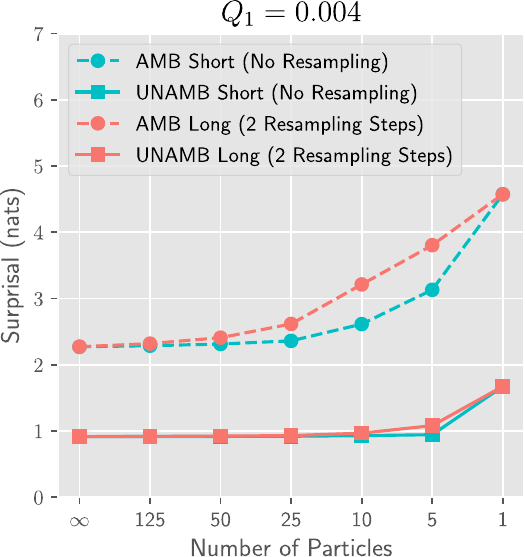} \hfill
    \includegraphics[width=0.49\linewidth]{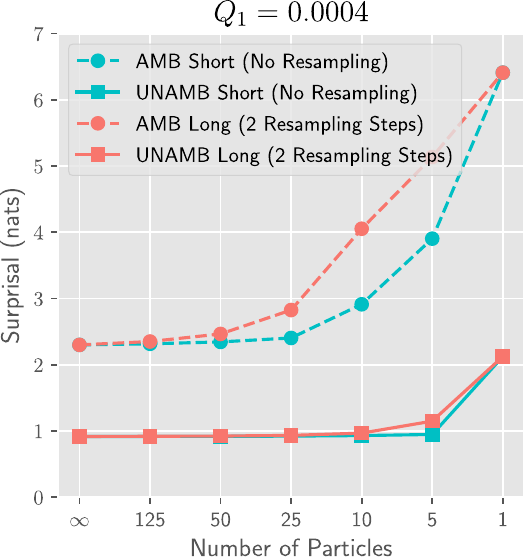} \\ 
    \vspace{2mm}
    \includegraphics[width=0.49\linewidth]{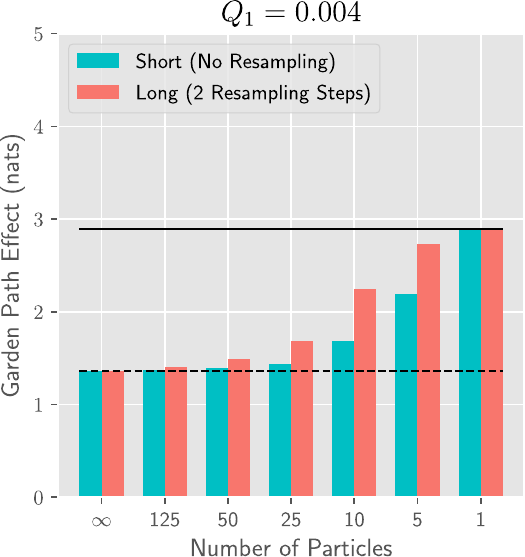} \hfill
    \includegraphics[width=0.49\linewidth]{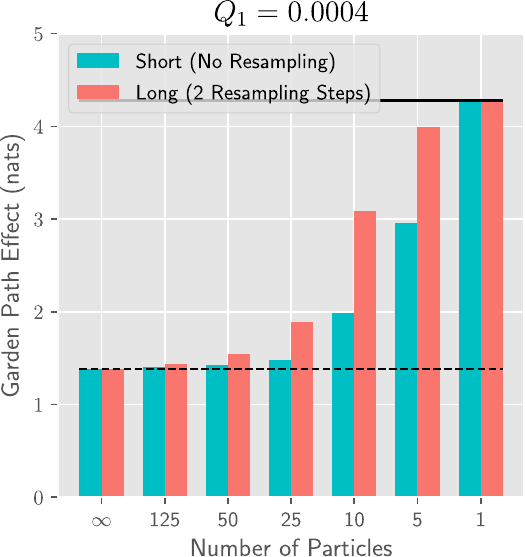}
    \caption{Expected surprisal (top) and garden-path effects (bottom) for a hypothetical digging-in experiment with short and long ambiguous regions, varying disambiguation strength (through $Q_1 = Q(w \mid T_1, C)$) and particles $N$. Long versions typically add 2-3 words and so we use 2 resampling steps. Fixed parameters: $Q(w\mid T_2,C) = 0.5$, \mbox{$\pi_{\text{AMB}}(T_1\mid C) = 0.8$}, \mbox{$\pi_{\text{UNAMB}}(T_1\mid C) = 0.2$} (\textit{where preferred is correct}). Using 50 thousand trials; 95\% error bars would not be visible.}
    \label{fig:dig}
\end{figure}

\section{Theoretical Analysis}

This section provides the mathematical foundations for the aforementioned results. We begin by asking: what happens to the expected surprisal when we resample the particles?

\begin{restatable}[Expected Surprisal Under Resampling]{thm}{surprisalUnderResampling}\label{thm:surprisal_under_resampling}
    The expected surprisal of any word in context increases monotonically as a function of resampling steps. That is, $$\E_{\p[i+1]}\left[S\left(\p[i+1]\right)\right] \geq \E_{\p}\left[S\left(\p\right)\right],$$
    Additionally, the above inequality is strict if the following statement holds: $$\E_{\p}\left[\Var\left(\m[i+1]\mid \p\right)\right] > 0.$$
\end{restatable}

More intuitively, if changing the distribution over structures also changes the probability of the word, and the distribution over structures is able to be changed, then \textbf{resampling will always increase the expected surprisal of the word}. This result also holds when sampling the initial $\p[0]$ from $\pi$ (which we can view as $\p[-1]$). Next, we ask: what happens if we continue this process of resampling forever?

\begin{restatable}[Cost of Repeated Resampling]{thm}{costRepeatedResampling}\label{thm:cost_repeated_resampling}
Resampling indefinitely introduces a cost that is associated with ignoring the structural information present in the word. Define \mbox{$\p[\infty] \coloneqq \lim_{i \to \infty} \p$}. Then 
\begin{align*}
    \E_{\p[\infty]}\left[S\left(\p[\infty]\right)\mid \p\right]& - S\left(\p\right)
    = \\* 
    D_{KL}&\left(\p(T\mid C) \,\|\, \p(T\mid w,C) \right),
\intertext{where, following Bayes' rule,}
    \p(T\mid w,C) \coloneqq& \frac{Q(w\mid T,C)\p(T\mid C)}{\sum_{T^* \in \mathcal{T}}Q(w\mid T^*,C)\p(T^*\mid C)}.
\end{align*}
\end{restatable}

Starting from any point in the resampling process, expected surprisal monotonically increases, eventually reflecting total disregard for the structural information in $w$. \textbf{This cost is greatest when the word disambiguates to a dispreferred structure.} At the population level, resampling appears to strengthen commitment to the structural prior---an increasing unwillingness to update beliefs. However, this is somewhat misleading: at the individual level, each sample entertains fewer analyses and becomes restricted in its ability to update, eventually committing to a single structure drawn from the prior. If any prior structures assign the word zero probability, causing total parse failure for some individuals, the expected surprisal increase is infinite.

\Cref{thm:surprisal_under_resampling,thm:cost_repeated_resampling} hint at digging-in effects, but demonstrating them requires showing that expected surprisal increases \textit{faster} for words disambiguating to dispreferred structures. To make this step, we next analyze the \textit{surprisal delta}:
\begin{equation}
    \Delta S(i) \coloneqq \E_{\p[i+1]}\left[S\left(\p[i+1]\right)\right] - \E_{\p}\left[S\left(\p\right)\right],
\end{equation}
the increase in expected surprisal from resampling step $i$ to $i+1$. In order to provide an intuition for digging-in effects, we work with two interpretable approximations of $\Delta S(i)$.

\begin{restatable}[Second-Order Approximation of Surprisal Delta]{thm}{secondOrderDelta}\label{thm:second_order_delta}
If the second-order Taylor series provides a good approximation of \mbox{$f(x) = -\log(x)$} around \mbox{$x = \m$} for $N$ particles, then the surprisal delta can be estimated as
$$\Delta S(i) \approx \frac{1}{2N} \E_{\p}\left[CV_{\p}(Q)^2\right],$$
where \mbox{$CV_{\p}(Q)$} is the coefficient of variation of word probability \mbox{$Q(w\mid T,C)$} across structures under \mbox{$\p(T\mid C)$}.
\end{restatable}

The coefficient of variation measures the standard deviation relative to the mean. Variability of $Q$ is high when structures disagree about the probability of $w$, and the mean is low when $w$ is unlikely under dominant structures---so \mbox{$CV_{\p}\big(Q\big)$} is maximized precisely for words disambiguating to dispreferred structures. The effect attenuates with particle count $N$, recovering the intuition that \textbf{the fully parallel surprisal theory does not predict language-independent digging-in effects.}

Our second approximation avoids the poor performance of the Taylor series for small values (unlikely words), instead dividing the total expected surprisal increase (\Cref{thm:cost_repeated_resampling}) by expected convergence time.

\begin{figure*}[htbp]
    \centering
    \includegraphics[width=\linewidth]{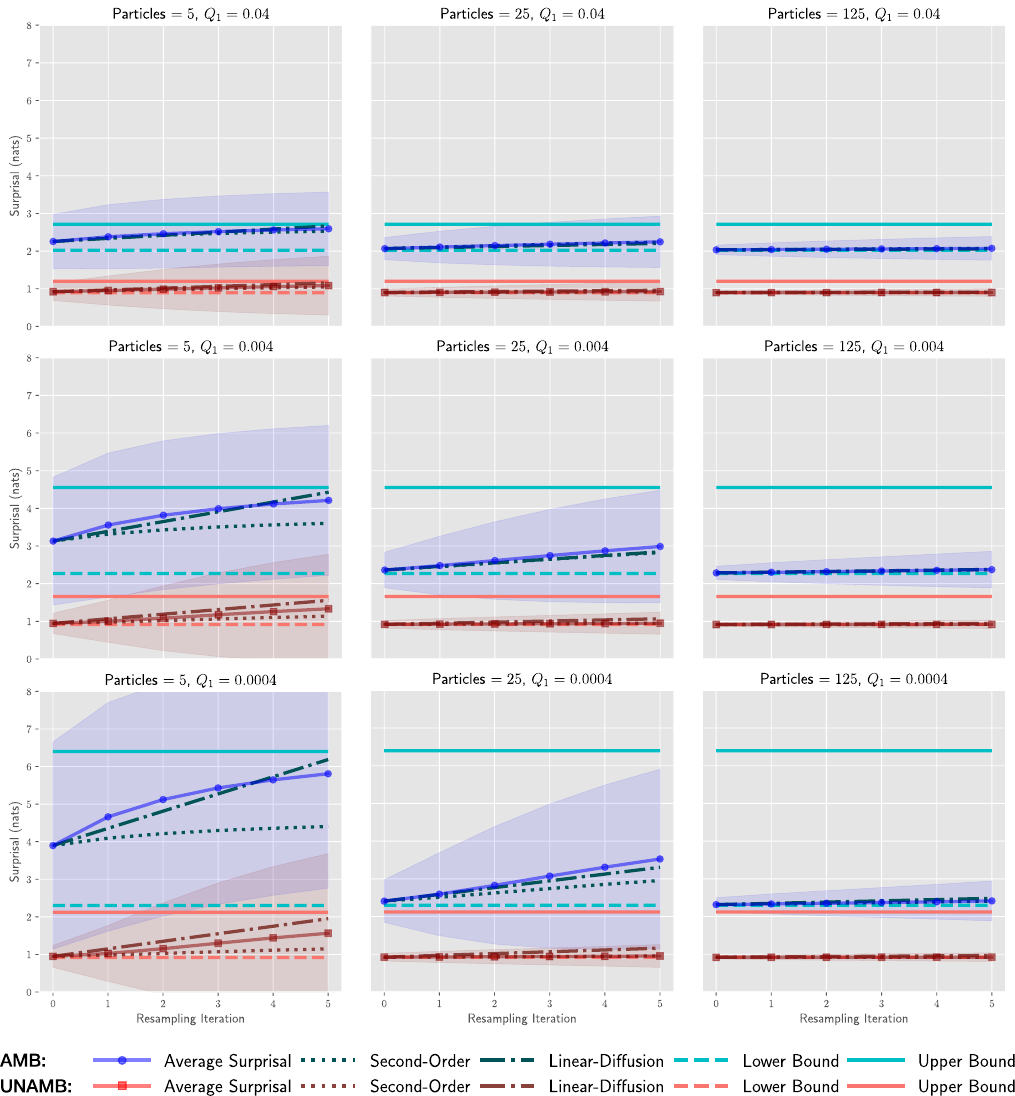}
    \caption{Expected surprisal in early resampling steps, varying disambiguation strength (through $Q_1 = Q(w \mid T_1, C)$) and the number of particles $N$. Fixed parameters: $Q(w\mid T_2,C) = 0.5$, \mbox{$\pi_{\text{AMB}}(T_1\mid C) = 0.8$}, \mbox{$\pi_{\text{UNAMB}}(T_1\mid C) = 0.2$} (again, \textit{this is not unambiguous in the traditional sense}: if there is truly only one compatible structure, then the expected surprisal is fixed). The second-order approximation is computed at each empirical sample of $\p$ and cumulatively summed from the empirical starting point \mbox{$\E_{\p[0]}\left[S\left(\p[0]\right)\right]$}. The linear-diffusion approximation applies a constant slope estimated from the empirical sample of $\p[0]$. Shaded regions indicate \textpm1 stdev of the sample (1 million trials); 95\% error bars would not exceed the line width.}
    \label{fig:grid}
\end{figure*}

\begin{restatable}[Linear-Diffusion Approximation of Surprisal Delta]{thm}{linearDiffusionDelta}\label{thm:linear_diffusion_delta}
If continuous diffusion dynamics provide a good approximation of the discrete multinomial resampling process for $N$ particles and the average surprisal delta is representative for the step of interest, then \mbox{$\Delta S(i)$} is approximately
\begin{equation*}
    \frac{1}{2N}\E_{\p[0]}\left[\frac{D_{KL}\left(\p[0](T\mid C) \,\|\, \p[0](T\mid w,C) \right)}{\sum_{T \in \mathcal{T}}\left(\p[0](T\mid C) - 1\right)\ln \left(1-\p[0](T\mid C)\right)}\right].
\end{equation*}
If \mbox{$\p[0](T\mid C) = 1$} for any structure $T$, we define the entire expression as zero, because only one structure is maintained and therefore the expected surprisal is fixed.
\end{restatable}

In the \textit{binary case} of two structures, the denominator equals the Shannon entropy of $\p[0]$, so \textbf{expected surprisal increases faster when there is high certainty for the wrong structure.} For the general case, the denominator is independent of the word $w$, so words with more structural information and thus higher cost must show faster surprisal increases. Since words that disambiguate to dispreferred structures have more information, this produces digging-in effects that attenuate with particle count $N$---matching the second-order approximation.

Using a simplified example, \Cref{fig:grid} visualizes the trajectory of expected surprisal alongside the two approximations for early resampling steps (\mbox{$\Delta S(i)$} for \mbox{$i \in [0, 4]$})---the case most relevant for digging-in effects---varying the disambiguation strength and the number of particles. Using the same data, we then investigate how well our approximations fit the true surprisal delta. \Cref{fig:fits} shows how well our approximations predict the true surprisal delta, visualized on both linear and log-log scales. In terms of magnitude, the linear-diffusion approximation better accounts for large increases, while the second-order approximation better accounts for small increases. In terms of correlation, the second-order approximation is significantly better in both variance explained (Pearson $r^2$: second-order = 0.82; linear-diffusion = 0.70) and rank order (Spearman $\rho$: second-order = 0.99; linear-diffusion = 0.89).\\\par

\begin{figure}[t]
    \centering
    \includegraphics[width=\linewidth]{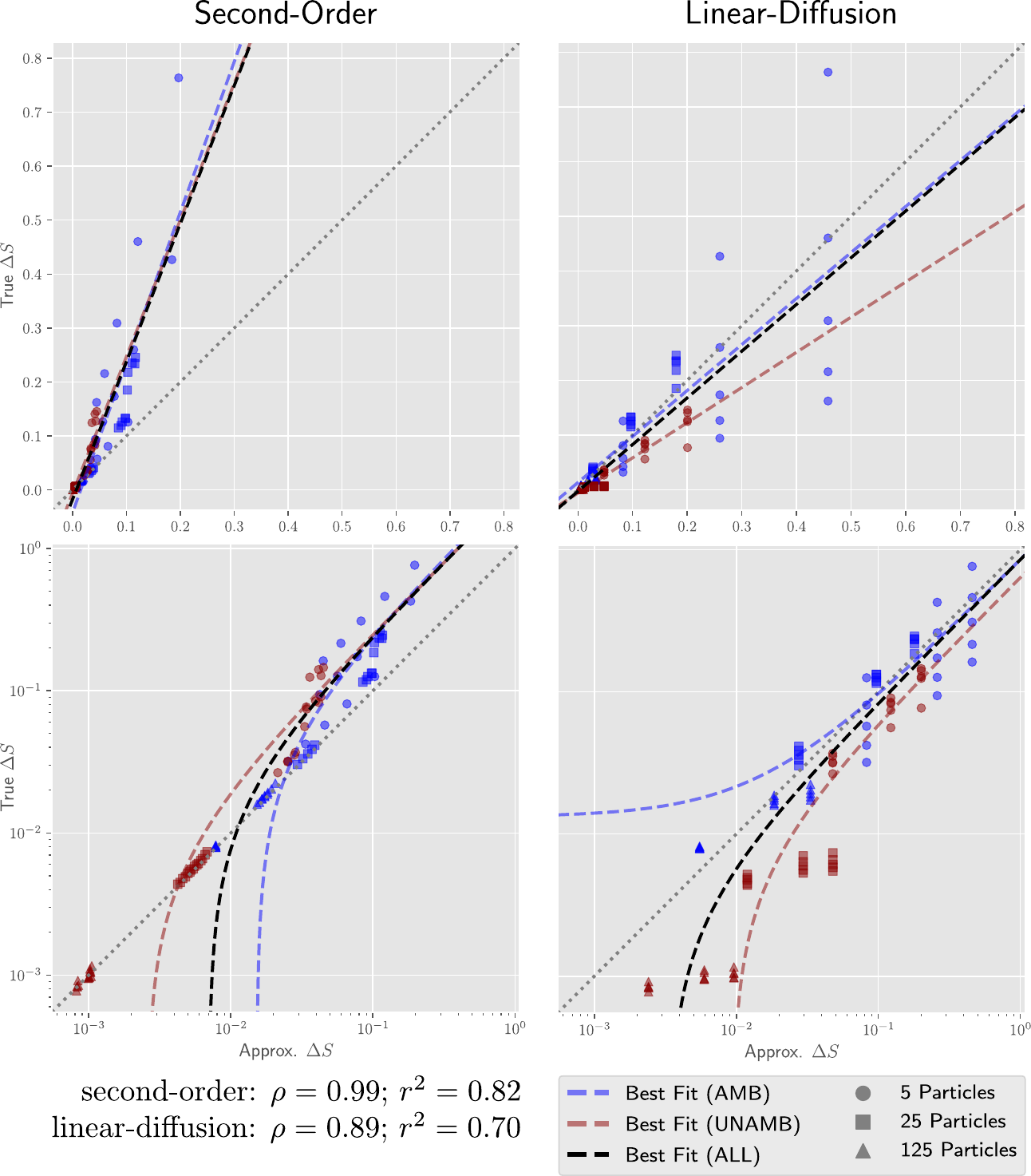}
    \caption{True per-step surprisal increase by approximated value, using the data from \Cref{fig:grid}. Top: visualized on linear scale. Bottom: visualized on log-log scale.}
    \label{fig:fits}
\end{figure}

Mathematically inclined readers may also be interested in \citet{Huggins2014AnIA}, which presents a measure-theoretic analysis of resampling in sequential Monte Carlo methods and describes, more generally, cases where the algorithm without resampling (or with adaptive resampling) might be preferred.

\section{Conclusion}

In this work, we analyzed the algorithmic consequences of particle filter models, focusing on the relationship between resampling and the expected surprisal of a word. We proved that, in the absence of new information, expected surprisal increases monotonically with each resampling step and asymptotically increases by the KL divergence from the prior to the posterior over structures---a cost greatest when words disambiguate to dispreferred structures. This analysis predicts amplified garden-path effects and offers a candidate principled account---within surprisal theory---of why garden-path disambiguation costs are greater than predicted by LLM surprisal estimates \citep{vanSchijndel2020SingleStagePM,wilcox-etal-2021-targeted,Huang2024LargescaleBY}, due to computational constraints on the parallelism of analysis.

We then demonstrated how resampling produces digging-in effects through two approximations of the increase in expected surprisal. Simulations across particle counts and disambiguation strengths validate both approximations and thus their intuitions. In both approximations, digging-in magnitude is inversely proportional to particle count---as expected, fully parallel models predict no digging-in. This analysis provides mathematical grounding for the empirical demonstration of \citet{Levy2008ModelingTE} of digging-in effects in particle filter parsing with resampling, and extends to cases where garden-path disambiguation does not involve parse failure.

These results have important implications for limitedly parallel models of sentence processing. If empirical evidence fails to support robust real-time digging-in effects, this would argue against models with regular resampling---the set of entertained analyses would need to be selectively updated, keeping low-probability parses around for longer than may seem rational. Conversely, robust digging-in effects could help constrain estimates of the degree of parallelism in human parsing.

In principle, the prediction of digging-in effects under surprisal theory should generalize to any model in which analyses may be discarded incrementally, though the specific mathematical derivations presented here may no longer apply. In particle filters, digging-in arises even under stochastic drift due to the nonlinear relationship between probability and surprisal. Models that impose explicit pressure to retain only high-probability analyses, such as beam search, should therefore exhibit equal or stronger digging-in effects, as low-probability parses are intentionally eliminated.

Finally, since our analysis assumes uninformative ambiguous regions, the results should always be interpreted relative to a fully parallel baseline. Our analysis demonstrates that longer ambiguous regions should incur \textit{additional} difficulty beyond what would be predicted by language statistics. Translating these algorithmic results into empirical predictions depends on additional factors not estimated here.

\vspace{3.6mm}

\section{Acknowledgments}

We thank the anonymous HSP 2026 reviewers for their helpful feedback. This work was supported by a grant from the Simons Foundation to the Simons Center for the Social Brain at MIT. Additionally, AMK is supported by the Fannie and John Hertz Foundation and an MIT Dean of Science Fellowship.

\vspace{3.6mm}

\printbibliography
\appendix
\newpage

\section{Proofs}

\surprisalUnderResampling*

\begin{proof}
    We begin with the definition of the surprisal delta:
    \begin{equation}
        \Delta S(i) = \E_{\p[i+1]}\left[S\left(\p[i+1]\right)\right] - \E_{\p}\left[S\left(\p\right)\right].
    \end{equation}
    Using the Law of Total Expectation \citep{casella2002statistical}, we can rewrite the first term with an expectation over the previous sample $\p$:
    \begin{align}
        &= \E_{\p}\left[\E_{\p[i+1]}\left[S\left(\p[i+1]\right)\mid \p \right]\right] - \E_{\p}\left[S\left(\p\right)\right] \\
        &= \E_{\p}\left[\E_{\p[i+1]}\left[S\left(\p[i+1]\right)\mid \p \right] - S\left(\p\right)\right].
    \intertext{We then rewrite $S$ using its definition:}
        &= \E_{\p}\left[\E_{\p[i+1]}\left[-\log \m[i+1]\mid \p \right] + \log \m\right]. \label{eq:delta_s_expanded}
    \end{align}
    Since \mbox{$f(x) = -\log x$} is strictly convex with \mbox{$f''(x) = \sfrac{1}{x^2} > 0$}, we can apply Jensen's Inequality \citep{casella2002statistical} to the random variable \mbox{$\m[i+1]\mid \p$}. Here, Jensen's Inequality tells us that \begin{equation}\label{eq:jensen_application}\E_{\p[i+1]}\left[-\log \m[i+1]\mid \p \right] \geq -\log\E_{\p[i+1]}\left[\m[i+1]\mid \p \right],\end{equation} and so we have that
    \begin{equation}
        \Delta S(i) \geq \E_{\p}\bigg[-\log\E_{\p[i+1]}\left[\m[i+1]\mid \p \right] + \log \m\bigg].\label{eq:after_jensen}
    \end{equation}
    By the linearity of $M$ in $\pi$ and the unbiasedness of the sample mean \citep{casella2002statistical}, we have that \mbox{$\E_{\p[i+1]}\left[\m[i+1]\mid \p \right] = \m$}. Therefore
    \begin{equation}
        \Delta S(i) \geq \E_{\p}\bigg[-\log\m + \log \m\bigg] = 0.
    \end{equation}
    Since \mbox{$f(x) = -\log x$} is strictly convex, the application of Jensen's Inequality in \eqref{eq:jensen_application} is strict whenever the random variable is not almost surely constant \citep{taboga2021jensen}, that is, \mbox{$\Var\left(\m[i+1]\mid \p\right) > 0$}. If there is variance for at least one particle set $\p$ that has nonzero probability, then the expression from \eqref{eq:after_jensen} onward has a strict inequality. Equivalently, the inequality is strict when the expected variance is nonzero: \mbox{$\Delta S(i) > 0$} whenever \begin{equation}
        \E_{\p}\left[\Var\left(\m[i+1]\mid \p\right)\right] > 0.
    \end{equation}
    If there is no variability in the word's probability for the subsequent time step, then the expected surprisal is fixed.
\end{proof}

\begin{lem}[Maximum Expected Surprisal Under Resampling]\label{eq:max_surprisal}
    Define \mbox{$\p[\infty] \coloneqq \lim_{i \to \infty} \p$}. If we condition on a specific distribution within the sequence, $\p[k]$, then 
    \begin{align*}
    \E_{\p[\infty]}\left[S\left(\p[\infty]\right)\mid \p[k]\right]& = \\* \sum_{T \in \mathcal{T}} -&\log Q(w\mid T,C) \p[k](T\mid C).
    \end{align*}
\end{lem}

\begin{proof}
Let \mbox{$\pj \coloneqq \p[k+j]$} for \mbox{$j \geq 0$} be a sequence of distributions over structures formed by iteratively resampling $j$ times from an initial distribution \mbox{$\p[k](T\mid C)$}. For any structure $T$, this sequence satisfies the three conditions of a Martingale \citep{durrett2019probability}: (i) \mbox{$\E\left[\lvert \pj(T\mid C) \rvert \right] < \infty$}, by nature of being a probability; (ii) the sequence is adapted to the natural filtration generated by the resampling steps; and (iii) \mbox{$\E\left[\pj[j+1](T\mid C)\right] = \pj(T\mid C)$}, by the unbiasedness of the sample mean. By the Martingale Convergence Theorem \citep{durrett2019probability,doob1953stochastic}, this sequence converges to a random variable \mbox{$\pj[j\to\infty](T\mid C)$} with finite expectation. Convergence requires that the variance of the resampling step vanishes, which occurs only when \mbox{$\pj(T\mid C) = 0$} or \mbox{$\pj(T\mid C) = 1$}. Since the sequence must converge, \mbox{$\pj[j\to\infty](T\mid C) \in \{0, 1\}$} almost surely. However, as a Martingale, expectation is preserved throughout, and---since the sequence is bounded in \mbox{$[0,1]$} and thus uniformly integrable---this extends to the limit:
\begin{equation}
    \E_{\pj[j\to\infty]}\left[\pj[j\to\infty](T\mid C)\right] = \pj[0](T\mid C) = \p[k](T\mid C).
\end{equation}
Since \mbox{$\pj[j\to\infty](T\mid C) \in \{0, 1\}$}, it must be the case that the event \mbox{$\{\pj[j\to\infty](T\mid C) = 1\}$} occurs with probability \mbox{$\p[k](T\mid C)$}. Let $\pj[j\to\infty,T]$ refer to the $\pj[j\to\infty]$ that assigns \mbox{$\pj[j\to\infty](T\mid C) = 1$} and occurs with probability
\begin{equation}\label{eq:probj}
    P\left(\pj[j\to\infty] = \pj[j\to\infty,T]\right) = \p[k](T\mid C).
\end{equation}
With this, we can derive an expression for the maximum expected surprisal, conditioned on a previous distribution $\p[k]$:
\begin{equation}
    \E_{\p[\infty]}\left[S\left(\p[\infty]\right)\mid \p[k]\right] \coloneqq \E_{\pj[j\to\infty]}\left[S\left(\pj[j\to\infty]\right)\right].
\end{equation}
Rewriting the right-hand expectation, we have
\begin{equation}
    \E_{\p[\infty]}\left[S\left(\p[\infty]\right)\mid \p[k]\right] = \sum_{\pj[j\to\infty]}S\left(\pj[j\to\infty]\right) P\left(\pj[j\to\infty]\right).
\end{equation}
We can then re-index the expectation using the structures $T$ that each $\pj[j\to\infty]$ place their probability mass on:
\begin{align}
    &= \sum_{T \in \mathcal{T}}S\left(\pj[j\to\infty,T]\right)P\left(\pj[j\to\infty] = \pj[j\to\infty,T]\right),
\intertext{which, by \eqref{eq:probj}, is}
    &= \sum_{T \in \mathcal{T}} S\left(\pj[j\to\infty,T]\right) \p[k](T\mid C).
\end{align}
Focusing on the surprisal term, by definition we have that $S\left(\pj[j\to\infty,T]\right)$ is equal to
\begin{align}
    -&\log \sum_{T^* \in \mathcal{T}}Q(w\mid T^*,C)\pj[j\to\infty,T](T^*\mid C).
\intertext{Since $\pj[j\to\infty,T]$ puts all mass on $T$, it functions as an indicator variable for that structure,}
    = -&\log \sum_{T^* \in \mathcal{T}} Q(w\mid T^*,C)\mathds{1}_{\{T^* = T\}},
\intertext{and selects out only that term from the sum,}
    = -&\log Q(w\mid T,C).
\end{align}
Substituting this back in, we obtain the desired result
\begin{align}
    \E_{\p[\infty]}\left[S\left(\p[\infty]\right)\mid \p[k]\right]& = \nonumber \\* 
    \sum_{T \in \mathcal{T}} -&\log Q(w\mid T,C) \p[k](T\mid C).
\end{align}
This expression is closely related to the initial surprisal. While $S\left(\p[k]\right)$ applies $-\log$ to the expectation of $Q$ over structures, \mbox{$\E_{\p[\infty]}\left[S\left(\p[\infty]\right)\mid \p[k]\right]$} is the expectation of $-\log Q$.
\end{proof}

\costRepeatedResampling*

\begin{proof}
Applying \Cref{eq:max_surprisal}, we have that
\begin{align}
    \E_{\p[\infty]}\left[S\left(\p[\infty]\right)\mid \p\right]& - S\left(\p\right) = \nonumber \\*
    \sum_{T \in \mathcal{T}} -\log Q(w\mid T,&C) \p(T\mid C) - S\left(\p\right).
\end{align}
We next rewrite the surprisal term,
\begin{align}
    &=\sum_{T \in \mathcal{T}} -\log Q(w\mid T,C) \p(T\mid C) + \log \m, \\
\intertext{and use the fact that \mbox{$\sum_{T \in \mathcal{T}}\p(T\mid C) = 1$} to combine the terms:}
    &=\sum_{T \in \mathcal{T}} -\log Q(w\mid T,C) \p(T\mid C) \nonumber \\* 
    &\hspace{1in}+\sum_{T \in \mathcal{T}} \p(T\mid C) \log \m \\
    &= \sum_{T \in \mathcal{T}}  \p(T\mid C) \left(\log \m - \log  Q(w\mid T,C)\right) \\
    &= \sum_{T \in \mathcal{T}} \p(T\mid C) \log \left(\frac{\m}{Q(w\mid T,C)} \right).
\end{align}
We now focus on just the inside of the log term. Expanding $\m$ using its definition, we obtain
\begin{equation}
    \frac{\sum_{T^* \in \mathcal{T}} Q(w\mid T^*,C) \p(T^*\mid C)}{Q(w\mid T,C)}.
\end{equation}
Then, we can insert a multiplicative factor of \mbox{$\frac{\p(T\mid C)}{\p(T\mid C)}$}. Note that cases where \mbox{$\p(T\mid C) = 0$} can be safely ignored due to the factor of \mbox{$\p(T\mid C)$} outside of the log which would render the entire term zero.
\begin{equation}
    = \frac{\p(T\mid C)\sum_{T^* \in \mathcal{T}} Q(w\mid T^*,C) \p(T^*\mid C)}{\p(T\mid C)Q(w\mid T,C)}.
\end{equation}
Following Bayes' rule, we define
\begin{equation}
\p(T\mid w,C) \coloneqq \frac{Q(w\mid T,C)\p(T\mid C)}{\sum_{T^* \in \mathcal{T}}Q(w\mid T^*,C)\p(T^*\mid C)},
\end{equation}
and substitute this into the previous expression, obtaining
\begin{equation}
    = \frac{\p(T\mid C)}{\p(T\mid w,C)}.
\end{equation}
Putting everything back together, we have that
\begin{align}
    \E_{\p[\infty]}\Big[S\Big(\p[\infty]&\Big)\mid \p\Big] - S\left(\p\right) = \nonumber \\*
    &\sum_{T \in \mathcal{T}} \p(T\mid C) \log \left(\frac{\p(T\mid C)}{\p(T\mid w,C)}\right).
\end{align}
Lastly, we recognize this expression as the Kullback--Leibler divergence and produce the result:
\begin{equation}
    = D_{KL}\left(\p(T\mid C) \,\|\, \p(T\mid w,C) \right).
\end{equation}
Thus, resampling introduces a cost that is associated with ignoring the structural information present in the word. Since each $\p[\infty]$ has committed to a single structure, observing the word produces no structural belief updates.
\end{proof}

\begin{lem}[Second-Order Approximation of Jensen Gap]\label{lem:second_order_jensen}
    Let $X$ be a random variable with mean $\mu$ and variance $\sigma^2$, and let $f$ be a twice-differentiable function. If the higher-order terms of the Taylor expansion are negligible, then the Jensen gap can be approximated as
\begin{equation*}
    \E_X \left[f\Big(X\Big)\right] - f\left(\E_X \Big[X\Big]\right) \approx \frac{1}{2}f''(\mu)\sigma^2.
\end{equation*}
\end{lem}

\begin{proof}
We begin with a second-order Taylor series expansion of $f(X)$ around $\mu$:
\begin{equation}
    f(\mu) + f'(\mu)(X - \mu) + \frac{1}{2}f''(\mu)(X - \mu)^2.
\end{equation}
Now we apply the expectation, approximating $\E_X\Big[f(X)\Big]$ with
\begin{equation}
    \E_X\left[f(\mu) + f'(\mu)(X - \mu) + \frac{1}{2}f''(\mu)(X - \mu)^2\right].
\end{equation}
By applying linearity of the expectation and extracting constants, we have
\begin{equation}
    = f(\mu) + f'(\mu)\E_X\Big[(X - \mu)\Big] + \frac{1}{2}f''(\mu)\E_X\Big[(X - \mu)^2\Big]. \\
\end{equation}
Recognizing that \mbox{$\E_X\left[(X - \mu)^2\right] = \sigma^2$} and observing that \mbox{$\E_X\left[(X - \mu)\right] = 0$} produces a final approximation for the term:
\begin{equation}
   \E_X \left[f\Big(X\Big)\right] \approx f(\mu) + \frac{1}{2}f''(\mu)\sigma^2.
\end{equation}
Lastly, subtracting $f\left(\E_X \left[X\right]\right) = f(\mu)$ yields the result:
\begin{equation}
    \E_X \left[f\Big(X\Big)\right] - f\left(\E_X \Big[X\Big]\right) \approx \frac{1}{2}f''(\mu)\sigma^2.
\end{equation}
This second-order approximation reflects the intuition that the Jensen Gap is largest when there is high variance.
\end{proof}

\secondOrderDelta*

\begin{proof}
We begin from \eqref{eq:delta_s_expanded}, an earlier expression for the surprisal delta:
\begin{equation}
    \hspace{-0.8mm}\Delta S(i) = \E_{\p}\left[\E_{\p[i+1]}\left[-\log \m[i+1]\mid \p \right] + \log \m\right].
\end{equation}
Next, we rewrite $\m$ as \mbox{$\E_{\p[i+1]}\left[\m[i+1]\mid\p\right]$}, using the linearity of $M$ in $\pi$ and the unbiasedness of the sample mean. Within the outermost expectation, there is now a clear Jensen Gap for function \mbox{$f(x) = -\log x$} and random variable \mbox{$X = \m[i+1]\mid \p$},
\begin{equation}
    \E_{\p[i+1]}\left[-\log \m[i+1]\mid \p \right] + \log \E_{\p[i+1]}\left[\m[i+1]\mid\p\right],
\end{equation}
which we can approximate via \Cref{lem:second_order_jensen}. With \mbox{$f''(x) = \sfrac{1}{x^2}$}, \mbox{$\mu = \m$}, and \mbox{$\sigma^2 = \Var_{\p[i+1]}\left(\m[i+1]\mid \p\right)$}, we obtain the following approximation for the surprisal delta:
\begin{equation}
    \Delta S(i) \approx  \E_{\p}\left[\frac{1}{2} \cdot \frac{1}{\left(\m\right)^2} \cdot \Var_{\p[i+1]}\left(\m[i+1]\mid \p\right)\right].
\end{equation}
Let us now pause to determine what \mbox{$\Var_{\p[i+1]}\left(\m[i+1]\mid \p\right)$} is. Recall that \mbox{$\m[i+1]$} equals \mbox{$\sum_{T \in \mathcal{T}} Q(w\mid T, C)\p[i+1](T\mid C)$} and that $\p[i+1]$ is produced by sampling $N$ particles from $\p$. Consequently, $\m[i+1]$ is a sample mean of \mbox{$Q_n = Q(w\mid T_n,C)$}, where $T_n$ are drawn i.i.d. from $\p$:
\begin{equation}
    \Var_{\p[i+1]}\left(\m[i+1]\mid \p\right) = \Var_{\p}\left(\frac{1}{N}\sum_{n=1}^{N} Q_n\right).
\end{equation}
Because each $Q_n$ is independent and has identical variance, under properties of variance, the right-hand side simplifies to
\begin{equation}
    \Var_{\p}\left(\frac{1}{N}\sum_{n=1}^{N} Q_n\right) = \frac{1}{N^2}\sum_{n=1}^{N}\Var_{\p}\left(Q_n\right) = \frac{1}{N}\Var_{\p}\left(Q\right).
\end{equation}
Substituting this back in, we obtain
\begin{align}
    \Delta S(i) &\approx  \E_{\p}\left[\frac{1}{2} \cdot \frac{1}{\left(\m\right)^2} \cdot \frac{1}{N}\Var_{\p}\left(Q\right)\right] \\
    &\approx \frac{1}{2N}\E_{\p}\left[\frac{\Var_{\p}\left(Q\right)}{\left(\m\right)^2}\right].
\intertext{Rewriting $\m$ as the mean of $Q$ and recognizing that we have \mbox{$\sfrac{\sigma}{\mu}$}, the coefficient of variation, we produce the result:}
    &\approx \frac{1}{2N}\E_{\p}\left[\frac{\Var_{\p}\left(Q\right)}{\left(\E_{\p}\left[Q\right]\right)^2}\right] \\
    &\approx \frac{1}{2N}\E_{\p}\left[CV_{\p}(Q)^2\right].
\end{align}
Thus the surprisal delta grows with the expected coefficient of variation of $Q$ under $\p$ and attenuates with particles $N$.\looseness=-1
\end{proof}

\begin{restatable}[Expected Time to Convergence]{lem}{expectedTime}\label{lem:expected_time}
If continuous diffusion dynamics provide a good approximation of the discrete multinomial resampling process for $N$ particles, then the expected time $\tau$ for $\p[0]$ to converge to a single structure is
\begin{equation*}
    \tau\left(\p[0]\right) \hspace{-0.7mm} \approx \hspace{-0.5mm} -2N\hspace{-1mm}\sum_{T \in \mathcal{T}}\hspace{-0.5mm}\left(1-\p[0](T\mid C)\right)\ln \left(1-\p[0](T\mid C)\right).
\end{equation*}
\end{restatable}

\begin{proof}
In the neutral Wright-Fischer model of genetic drift as described in \citet{genes14091751}, there is a population of $N$ individuals with $A$ possible allele types. Let \mbox{$k_1, k_2, \ldots, k_A$} be the counts of the different alleles in the population with $\sum_{i=1}^A k_i = N$ and let \mbox{$f_1, f_2, \ldots, f_A$} be the corresponding frequencies $\sfrac{k_i}{N}$. In the discrete formulation of the process, each new generation is formed by sampling with replacement from the old generation. From equation (30) in \citet{genes14091751}, for a constant population size $N$ and $A_0$ initial alleles with initial frequencies \mbox{$f_{1,0}, f_{2,0}, \ldots, f_{A_0,0}$}---under the diffusion approximation of the process---the expected time to fixation of the last surviving allele is 
\begin{equation}
    -2N \sum_{i=1}^{A_0}(1-f_{i,0})\ln (1-f_{i,0}).
\end{equation}
This framework applies to ours by analogy: individuals $N$ correspond to particles $N$, alleles $A$ to structures $\lvert\mathcal{T}\rvert$, and frequencies $f_{i,0}$ to probabilities \mbox{$\p[0](T_i\mid C)$}. Applying their result to our context, we obtain an approximation for the expected time to convergence:
\begin{equation}
    -2N\sum_{T \in \mathcal{T}}\left(1-\p[0](T\mid C)\right)\ln \left(1-\p[0](T\mid C)\right).
\end{equation}
Convergence is the time at which all probability mass has been placed on a single structure. While the expression closely resembles a Shannon entropy over structures, \mbox{$H(T) = -\sum_{T\in\mathcal{T}}P(T)\ln P(T)$}, the term \mbox{$1 - \p[0](T\mid C)$} will only correspond to the probability of a structure in the binary case of two structures.
\end{proof}

\linearDiffusionDelta*

\setlength{\abovedisplayskip}{5pt}
\setlength{\belowdisplayskip}{5pt}

\begin{proof}
We create a linear approximation of the surprisal delta by dividing the total increase in surprisal from the starting point $\p[0]$ by its expected time to converge. We then take the expectation over possible starting points $\p[0]$, producing an approximate expected slope:
\begin{align}
    \Delta S(i) &\approx \E_{\p[0]}\left[\frac{\E_{\p[\infty]}\left[S\left(\p[\infty]\right) \mid \p[0]\right] - S\left(\p[0]\right)}{\tau\left(\p[0]\right)}\right].
\end{align}
By \Cref{thm:cost_repeated_resampling} and \Cref{lem:expected_time}, we have that this is equal to
\begin{equation}
\hspace{-1.7mm}\E_{\p[0]}\hspace{-1mm}\left[\hspace{-0.5mm}\frac{D_{KL}\left(\p[0](T\mid C) \,\|\, \p[0](T\mid w,C) \right)}{-2N\hspace{-0.5mm}\sum_{T \in \mathcal{T}}\hspace{-0.5mm}\left(\hspace{-0.2mm}1\hspace{-0.2mm}-\hspace{-0.2mm}\p[0](T\mid C)\hspace{-0.2mm}\right)\hspace{-0.5mm}\ln\hspace{-0.5mm}\left(\hspace{-0.2mm}1\hspace{-0.2mm}-\hspace{-0.2mm}\p[0](T\mid C)\hspace{-0.2mm}\right)}\hspace{-0.5mm}\right]\hspace{-0.8mm}.\hspace{-0.5mm}
\end{equation}
Extracting constants and simplifying, we obtain the result:
\begin{equation}
\hspace{-1.3mm}\frac{1}{2N}\hspace{-1mm}\E_{\p[0]}\hspace{-1mm}\left[\hspace{-0.5mm}\frac{D_{KL}\left(\p[0](T\mid C) \,\|\, \p[0](T\mid w,C) \right)}{\sum_{T \in \mathcal{T}}\hspace{-0.5mm}\left(\hspace{-0.2mm}\p[0](T\mid C)\hspace{-0.2mm}-\hspace{-0.2mm}1\hspace{-0.2mm}\right)\hspace{-0.5mm}\ln\hspace{-0.5mm}\left(\hspace{-0.2mm}1\hspace{-0.2mm}-\hspace{-0.2mm}\p[0](T\mid C)\hspace{-0.2mm}\right)}\hspace{-0.5mm}\right]\hspace{-0.5mm}.\hspace{-0.2mm}
\end{equation}
Since the structural convergence time is independent of the word $w$, words that have higher structural information cost will show larger per-step increases in expected surprisal.
\end{proof}

\end{document}